\documentclass[10pt,twocolumn,twoside]{IEEEtran}
\usepackage[T1]{fontenc}
\usepackage[dvipsnames]{xcolor}
\usepackage{amsmath,amssymb,amsthm,ascii,bm,bbm,capt-of,cite,color,epsfig,hyphenat,mathrsfs,subfigure,stfloats,mathtools,marvosym}
\usepackage[english]{babel}
\usepackage[utf8]{inputenc}
\usepackage{graphicx}
\usepackage{amsmath}
\usepackage{hyperref, bm}
\usepackage{marginnote}

\usepackage{mathtools}
\usepackage{color}
\usepackage{amsmath}
\usepackage{lipsum}
\usepackage{multirow}
\usepackage{cancel}
\usepackage{pgfplots}
\usepackage{cleveref}
\usepackage{outlines}
\usetikzlibrary{pgfplots.groupplots}
\usetikzlibrary{plotmarks}
\newlength\figH
\newlength\figW
\pgfplotsset{minor tick num=1,
    legend style={font=\scriptsize},
    every axis label/.append style={font=\scriptsize},
    yticklabel style={/pgf/number format/fixed},
    every axis x label/.style={at={(0.5,0)},yshift=-20pt},
    every axis y label/.style={at={(0,0.5)},xshift=-22pt,rotate=90},
    tick label style = {font=\scriptsize},
    every axis plot/.append style={line width=0.6pt},
    axis background/.append style={fill=grey!2},
    every axis/.append style={grid=both}
}

\usepackage{algpseudocode}
\usepackage{pifont}
\usepackage[ruled, noline]{algorithm2e}
\usepackage{upgreek}
\definecolor{grey}{gray}{0}
\usepackage[font=small,skip=4pt]{caption}

\title{Generalization error bounds for kernel matrix completion and extrapolation}
\newcommand{\trace}[1]{\text{\normalfont Tr}(#1)}
\newcommand{\es}[1]{\mathbb{E}_\sigma\left\{ #1\right\}}
\newcommand{\norm}[1]{\left|\left|#1\right|\right|^2_2}
\newcommand{\norms}[1]{\left|\left|#1\right|\right|_2}
\newcommand{\nuclear}[1]{\left|\left|#1\right|\right|_*}
\newcommand{\frob}[1]{\left|\left|#1\right| \right|_\text{F}^2}
\newcommand{\frobs}[1]{\left|\left|#1\right| \right|_\text{F}}

\def\m{\bm m}
\def\f{\bm f}

\def\bf{\textbf{f}}

\def\e{\bm e}
\def\E{\bm E}
\def\b{\bm b}
\def\bbE{\mathbb{E}}

\def\M{\bm M}
\def\H{\bm H}

\def\A{\bm A}
\def\B{\bm B}
\def\I{\bm I}

\def\R{\bm R}
\def\W{\bm W}

\def\C{\bm C}
\def\bbR{\mathbb{R}}
\def\mS{\mathcal{S}}
\def\mF{\mathcal{F}}
\def\S{\bm S}

\def\D{\bm D}

\def\F{\bm F}
\def\K{\bm K}
\def\bK{\bar{\bm{K}}}
\def\bgam{\bar{\bm{d}}}
\def\hbgam{\hat{\bar{\bm{d}}}}
\def\gam{\bm{d}}

\def\bphi{\bm \Phi}

\def\Sig{\bm{\Sigma}}

\newcommand{\argmin}{\arg\!\min}

\SetKwRepeat{Do}{do}{while}
\SetEndCharOfAlgoLine{}

\newtheorem{theorem}{Theorem}
\newtheorem{definition}{Definition}

\author{Pere~Gim\'enez-Febrer,
    Alba~Pag\`es-Zamora, and
    Georgios B. Giannakis
    \noindent\thanks{\hspace{-0.3cm}P. Giménez-Febrer and A. Pagès-Zamora are with the SPCOM Group, Universitat Polit\`ecnica de Catalunya-Barcelona Tech, Spain. \newline G. B. Giannakis is with the Dept.\ of ECE and Digital Technology Center, University of Minnesota, USA.\newline This work is supported by ERDF funds (TEC2013-41315-R and TEC2016-75067-C4-2), the Catalan Government (2017 SGR 578), and NSF grants (1500713, 1514056, 1711471 and 1509040).}
}

\begin{document}
\maketitle
\thispagestyle{plain}
\pagestyle{plain}
\begin{abstract}            
Prior information can be incorporated in matrix completion to improve estimation accuracy and extrapolate the missing entries. Reproducing kernel Hilbert spaces provide tools to leverage the said prior information, and derive more reliable algorithms. This paper analyzes the generalization error of such approaches, and presents numerical tests confirming the theoretical results.
\end{abstract}
\section{Introduction}
Matrix completion (MC) deals with the recovery of missing entries in a matrix -- a task emerging in several applications such as image restoration~\cite{ji2010robust}, collaborative filtering~\cite{rao2015collaborative} or positioning~\cite{nguyen2016}. MC relies on the low rank of data matrices to enable reliable, even exact~\cite{candes}, recovery of the full unknown matrix. Exploiting this property, mainstream approaches to MC involve the minimization of the nuclear norm~\cite{cai2010singular,ma2011fixed} or a surrogate involving the data matrix factorization into a product of two low-rank matrices~\cite{koren2009matrix,sun2015matrix}. 

One main assumption in the aforementioned approaches to MC is that the unknown matrix is incoherent, meaning the entries of its singular vectors are uniformly distributed, which implies that matrices with structured form are not allowed. For instance, data matrices with clustered form lead to segmented singular vectors that violate the incoherence assumption. Such structures may be induced by prior information embedded in, e.g., graphs~\cite{kalofolias2014matrix}, dictionaries~\cite{yi2015partial}, or heuristic assumptions~\cite{cheng2013stcdg}. Main approaches to MC leverage prior information with proper regularization~\cite{chen,gogna2015matrix,zhou2012,gimenez}, or, by restricting the solution space~\cite{jain2013,bazerque2013,gimenez2018,abernethy2006low3333}. Most of these approaches can be unified using a reproducing kernel Hilbert space (RKHS) framework~\cite{bazerque2013,gimenez2018}, which presents theoretical tools to exploit prior information.

When analyzing the performance of MC algorithms, several works, e.g.~\cite{candes2010matrix,cai2010singular,rao2015collaborative,jain2013}, focus on the derivation of sample complexity bounds; that is, the evolution of the distance to the optimum across the number of samples and iterations. Other analyses are based on the generalization error (GE)~\cite{shamir2014,srebro,foygel2011}, a metric that measures the difference between the value of the loss function applied to a training dataset, and its expected value~\cite{shawe}. When the probability distribution of the data is unknown, the expected value is replaced by the average loss on a testing dataset~\cite{yaniv2009}. Due to the potentially large matrix sizes and the small size of the training dataset, it is important that the estimated matrix exhibits low GE in order to prevent overfitting. 

In \cite{gimenez2018}, we introduced a novel Kronecker kernel matrix completion and extrapolation (KKMCEX) algorithm for MC. This algorithm relies on kernel ridge regression with the number of coefficients equal to the number of observations, thus being attractive for imputing matrices with a minimal number of observations. The present paper presents GE analysis for MC with prior information, and establishes that different from other MC approaches, the GE of KKMCEX does not depend on the matrix size, thus making it more reliable when dealing with few observations. 

\section{MC with prior information}
\label{sec:sideinfo}
%Consider a rank $r$ matrix $\F\in\bbR^{N\times L}$
 %which cannot be directly observed, and instead we observe its values through
 %\begin{equation}\label{eq:sigmodelmat}
    %\M = P_{\mS_m}(\F) + \E
%\end{equation}
%where $\mS_m\subseteq\{1,\ldots,N\}\times \{1,\ldots,L\}$ is the sampling set of cardinality $m=|\mS_m|$ containing the indices of the observed entries; $P_{\mS_m}(\cdot)$ is a projection operator that sets to zero the entries with index $(i,j)\notin \mS_m$ and leaves the rest unchanged; and, $\bm\E\in\mathbb{R}^{N\times L}$ is a noise matrix with $\E_{i,j} = 0 \: \forall \:(i,j)\:\notin\: \mS_m$. 
Consider a matrix $\M=\F+\E$, where $\F\in\bbR^{N\times L}$ denotes an unknown rank $r$ matrix, and $\E$ is a noise matrix.  We can only observe a subset of the entries in $\M$ whose indices are given by the sampling set $\mS_m\subseteq\{1,\ldots,N\}\times \{1,\ldots,L\}$ of cardinality $m=|\mS_m|$.  Factorizing the unknown matrix as $\F=\W\H$, where $\W\in\bbR^{N\times p}$, $\H\in\bbR^{L\times p}$ and $p\geq r$, the unknown entries can be recovered by estimating \vspace{-0.15cm}
\begin{align}   
    \!\{\hat{\W}\!,\!\hat{\H}\} \!=\! \argmin_{\!\!\!\!\!\!\!\!\!\!\!\substack{\W\in\mathbb{R}^{N\times p}\\\H\in\mathbb{R}^{L\times p}}}\! \frob{P_{\mS_m}\!(\!\M\!-\!\W\H^T)\!} \!+\! \mu\!\left(\frob{\W} \!+\! \frob{\H}\right)\raisetag{10pt}\label{eq:mclag}
\vspace{-0.30cm}
\end{align}
where $P_{\mS_m}(\cdot)$ denotes an operator that sets to zero the entries with index $(i,j)\notin \mS_m$ and leaves the rest unchanged, while $\mu$ is a regularization scalar. Hereafter we refer to \eqref{eq:mclag} as the base MC formulation, which can also be written with the nuclear norm as a regularizer through the property $\nuclear{\F} = \min_{\F=\W\H^T} {1\over 2}\left(\frob{\W} + \frob{\H}\right)$~\cite{srebro}.

While the basic MC formulation makes no use of prior information, kernel (K)MC incorporates such knowledge by means of kernel functions that measure similarities between points in their input spaces. Let $\mathcal{X}:=\{x_1,\ldots,x_N\}$ and $\mathcal{Y}:=\{y_1,\ldots,y_L\}$ be spaces of entities with one-to-one correspondence with the rows and columns of $\F$, respectively. Given the input spaces $\mathcal{X}$ and $\mathcal{Y}$, KMC defines the pair of RKHSs $\mathcal{H}_w := \left\{ w:\: w(x)=\sum\nolimits_{n=1}^{N}b_j\kappa_w(x,x_j), \: b_j\in\mathbb{R} \right\}$ and $ \mathcal{H}_h := \left\{ h:\: h(y)=\sum\nolimits_{l=1}^{L}c_j\kappa_h(y,y_j), \: c_j\in\mathbb{R} \right\}$, where $\kappa_w:\mathcal{X}\times\mathcal{X}\rightarrow\mathbb{R}$ and $\kappa_h:\mathcal{Y}\times\mathcal{Y}\rightarrow\mathbb{R}$ are kernel functions. Then, KMC postulates that the columns of the factor matrices in \eqref{eq:mclag} are functions in $\mathcal{H}_w$ and $\mathcal{H}_h$. Thus, we write $\W=\K_w\B$ and $\H=\K_h\C$,  where $\B$ and $\C$ are coefficient matrices, while $\K_w\in\mathbb{R}^{N\times N}$ and $\K_h\in\mathbb{R}^{L\times L}$ are the kernel matrices with entries $(\K_w)_{i,j}=\kappa_w(x_i,x_j)$ and $(\K_h)_{i,j}=\kappa_h(y_i,y_j)$. The KMC formulations proposed in~\cite{bazerque2013,zhou2012}, recover the factor matrices as
\begin{align}\label{eq:kmc}
\{\hat{\W}\!,\!\hat{\H}\} \!=\! \argmin_{\!\!\!\!\!\!\!\substack{\W\in\mathbb{R}^{N\times p}\\\H\in\mathbb{R}^{L\times p}}} &\frob{P_{\mS_m}(\!\M\!-\!\W\H^T)} \\[-1em]&\!+ \mu\!\left(\trace{\W^T\K_w^{-1}\W} \!+\! \trace{\H^T\K_h^{-1}\H}\right). \nonumber
\end{align}
The coefficient matrices are obtained as  $\hat{\B}\!=\!\K_w^{-1}\hat{\W}$ and $\hat{\C}\!=\!\K_h^{-1}\hat{\H}$, although this step is usually omitted~\cite{bazerque2013,zhou2012}. 

Algorithms solving \eqref{eq:mclag} and \eqref{eq:kmc} rely on alternating minimization schemes that do not converge to the optimum in a finite number of iterations~\cite{jain2013low}. To overcome this limitation and obtain a closed-form solution, we introduced the Kronecker kernel MC and extrapolation (KKMCEX) method~\cite{gimenez2018}. Associated with entries of $\F$, consider the two-dimensional function $f:\mathcal{X}\times\mathcal{Y}\rightarrow\mathbb{R}$ with $f(x_i,y_j) = \F_{i,j}$, and the RKHS it belongs to
\begin{equation}
    \mathcal{H}_f \!:= \!\left\{\!f: f(x,y)\!=\!\!\sum_{n=1}^{N}\sum_{l=1}^{L}d_{n,l}\kappa_f((x,x_n),(y,y_l)), d_{n,l}\!\in\mathbb{R}\! \right\} . \nonumber
\end{equation}
Upon vectorizing $\F$, we  obtain  $\f=\text{vec}(\F) =$ $\K_f\gam$,  where $\K_f$ has entries $\kappa_f$ and  $\gam :=[d_{1,1},\ldots,d_{N,1},\ldots,d_{N,L}]^T$. Accordingly, the data matrix is vectorized as $\overline{\m} = \S\text{vec}(\M)$, where $\S$ is an $m\times NL$ binary sampling matrix with a single nonzero entry per row, and $\bar{\e}=\S\text{vec}(\E)$ denotes the noise vector. With these definitions, the signal model for the observed entries becomes
%------------------------------------------------------------------------------
\begin{equation}\label{eq:sigmodelvec}
    \overline{\m} = \S\f + \bar{\e} = \S\K_f\gam + \bar{\e}.
\end{equation}
%------------------------------------------------------------------------------
Recovery of the vectorized matrix is then performed using the kernel ridge regression estimate of $\gam$ given by
%------------------------------------------------------------------------------
\begin{align}\label{eq:kkmcex}
    \hat{\gam} =& \argmin_{\gam\in\bbR^{NL}} \norm{\overline{\m} - \S\K_f\gam} + \mu\gam^T\K_f\gam.
\end{align}
%------------------------------------------------------------------------------
The closed-form solution to \eqref{eq:kkmcex} satisfies $\hat{\gam}=\S^T\hbgam$, where
\begin{equation}
    \hbgam=(\S\K_f\S^T+\mu\I)^{-1}\overline{\m} \label{eq:kkmcexsol}
\end{equation}
is the result of using the matrix inversion lemma on the solution to \eqref{eq:kkmcex}.
Since \eqref{eq:kkmcexsol} only depends on the observations in $\mS_m$, KKMCEX can be equivalently rewritten as
%------------------------------------------------------------------------------
\begin{align}\label{eq:kkmcexrep}
     \hbgam=\argmin_{\bgam\in\bbR^n} \norm{\overline{\m} - \bK_f\bgam} + \mu\bgam^T\bK_f\bgam 
\end{align}
%------------------------------------------------------------------------------
where $\bK_f =\S\K_f\S^T$.
Given $\kappa_w$ and $\kappa_h$, it becomes possible to use $\kappa_f((x,x_n),(y,y_l))=\kappa_w(x,x_n)\kappa_h(y,y_l)$ as a kernel, which corresponds to a kernel matrix $\K_f=\K_h\otimes\K_w$~\cite{gimenez2018}. 

\section{Generalization error in MC}
\label{sec:radbounds}

In this section, we derive bounds for the GE of the MC in~\eqref{eq:mclag}, KMC in~\eqref{eq:kmc} and KKMCEX in~\eqref{eq:kkmcex} algorithms. There are two approaches to GE analysis, namely the inductive~\cite{shawe} and the transductive one in~\cite{yaniv2009}. In the inductive one GE measures the difference between the expected value of a loss function and the empirical loss over a finite number of samples. Consider rewriting MC in the general form
%------------------------------------------------------------------------------
\begin{equation}
    \hat{\F}=\argmin_{\F\in\mathcal{F}} {1\over m} \sum\nolimits_{(i,j)\in\mS_m} l(\M_{i,j},\F_{i,j})
\end{equation}
%------------------------------------------------------------------------------
where $l:\bbR\times\bbR\rightarrow\bbR$ denotes the loss, and $\mathcal{F}$ is the hypothesis class. For instance, choosing the square loss and setting the class to the set of matrices with a nuclear norm smaller than a constant $t$ results in the base MC formulation~\eqref{eq:mclag}. Assuming a sampling distribution $\mathcal{D}$ over $\{1,\ldots,N\}\times\{1,\ldots,L\}$ for the observed indices in $\mS_m$, the GE for a specific estimate $\hat{\F}$ is given by the expected difference $\bbE_\mathcal{D}\{l(\M_{i,j},\hat{\F}_{i,j})\} - {(1/m)}\sum_{(i,j)\in\mS_m}l(\M_{i,j},\hat{\F}_{i,j})$. However, this definition of GE does not fit the MC framework because it assumes that: i) the data distribution is known; and, ii) the entries are sampled with repetition. In order to come up with distribution-free claims for MC, one may resort to the transductive GE analysis~\cite{yaniv2009}. In this scenario, we are given $\mS_n=\mS_m\cup\mS_u$ of $n$ data comprising the union of the training set $\mS_m$ and the testing set $\mS_u$, where $|\mS_u|=u$. These data are taken without repetition, and the objective is to minimize the loss on the testing set. Thus, the GE is the difference between the testing and training loss functions\vspace{-0.2cm}
%------------------------------------------------------------------------------
\begin{equation} \label{eq:tge}\\[-0.00cm]
    {1 \over u} \sum\nolimits_{(i,j)\in\mS_u} \hspace{-0.5em}l(\M_{i,j},\hat{\F}_{i,j}) - {1 \over m} \sum\nolimits_{(i,j)\in\mS_m} \hspace{-0.5em} l(\M_{i,j},\hat{\F}_{i,j}).\\[-0.00cm]
\end{equation}
%------------------------------------------------------------------------------
By making this difference as small as possible, we ensure that the chosen $\hat{\F}$ has good generalization properties, meaning we expect to obtain a similar empirical loss when we choose a different testing set of samples. Since MC algorithms find their solution among a class of matrices under different restrictions or hypotheses, we are interested in bounding \eqref{eq:tge} for any matrix in the solution space. Before we present such bounds, we need to introduce the notion of transductive Rademacher complexity (TRC) as follows.
\begin{definition} \textbf{Transductive Rademacher complexity}\cite{yaniv2009} Given a set $\mS_n=\mS_m\cup\mS_u$ with $q :={1 \over u}+{1\over m}$, the transductive Rademacher complexity (TRC) of a matrix class $\mF$ is
\vspace{-0.1cm}
    %------------------------------------------------------------------------------
    \begin{equation}\label{eq:trc}
        R_n(\mF) = q\es{\sup_{\F\in\mF}\sum\nolimits_{(i,j)\in\mS_n}\sigma_{i,j}\F_{i,j}}\\[-0.1cm] 
    \end{equation}
    %------------------------------------------------------------------------------
where $\sigma_{i,j}$ is a Rademacher random variable that takes values $[-1,1]$ with probability $0.5$. We may also write \eqref{eq:trc} in vectorized form as $R_n(\mF)= q\es{\sup_{\F\in\mF}\bm \sigma^T\text{\normalfont vec}(\F)}$, where $\bm \sigma = \text{vec}(\Sig)$, and $\Sig\in\mathbb{R}^{N\times L}$ has entries  $\Sig_{i,j} = \sigma_{i,j}$ if $(i,j)\in\mS_n$, and $\Sig_{i,j} = 0$ otherwise.
\end{definition}
TRC measures the expected maximum correlation between any function in the class and the random vector $\bm \sigma$. Intuitively, the greater this correlation is, the larger is the chance of finding a solution in the hypothesis class that will fit any observation draw, that is, $\hat{\F}_{i,j} \!\simeq\! \M_{i,j} \forall \: (i,j)\in\mS_n$. Although TRC measures the ability to fit both the testing and training data at once, a model for $\F$ is learnt using only the training data. While having a small loss across all entries in $\mS_n$ is desirable, making it too small can lead to overfitting, and an increased error when predicting entries outside $\mS_n$. Using the TRC, the GE is bounded as follows.
\begin{theorem}\label{th:tgebound}\cite{yaniv2009} Let $\mF$ be a matrix hypothesis class. For a loss function $l$ with Lipschitz constant $\gamma$, and any $\F\in\mF$, it holds with probability $1-\delta$ that 
   %------------------------------------------------------------------------------
    \begin{align}
        &{1 \over u} \sum\nolimits_{(i,j)\in\mS_u} l(\M_{i,j},\F_{i,j}) - {1 \over m} \sum\nolimits_{(i,j)\in\mS_m}l(\M_{i,j},\F_{i,j}) \nonumber \\ &\leq R_n(l\circ \mF) + 5.05q\sqrt{\min(m,u)}+ \sqrt{{2q}\ln{(1 / \delta)}}\;. \label{eq:tgebound}\\[-0.8cm]\nonumber
    \end{align}
    %------------------------------------------------------------------------------
\end{theorem}
Theorem~\ref{th:tgebound} asserts that in order to bound the GE, it only suffices to bound the TRC. Moreover, using the contraction property, which states that $R_n(l\circ\mathcal{F}) \leq {1\over \gamma} R_n(\mathcal{F})$~\cite{yaniv2009}, we only need to calculate the TRC of $\mF$. Given that the same loss function is used in MC, KMC and KKMCEX, in order to assess the GE upper bound of the three methods we will pursue the TRC for the hypothesis class of each algorithm.
\vspace{-0.1cm}
\subsection{Rademacher complexity for base MC}
In the base MC formulation~\eqref{eq:mclag}, the hypothesis class is $\mF_{MC}:=\{\F:\nuclear{\F}\leq t,\: t\in\bbR\}$, where the value of $t$ is regulated by $\mu$. As derived in~\cite{shamir2014}, the TRC for this class of matrices is bounded as\vspace{-0.00cm}
%------------------------------------------------------------------------------
\begin{equation}
    \!\!\!\! R_n(\mF_{MC}) \!\leq\! q \es{\!\!\!\!\sup_{\,\,\,\F\in\mF_{MC}}\!\!\! ||\Sig||_2\!\nuclear{\F}\!} \!\leq\! Gqt(\sqrt{N}\!\!+\!\!\sqrt{L})\label{eq:trcmcbound}\\[-0.00cm]
\end{equation}
%------------------------------------------------------------------------------
where $G$ is a universal constant.
The bound in \eqref{eq:trcmcbound} decays as $\mathcal{O}({1\over m}+{1\over u})\subseteq\mathcal{O}\left(1/\min(m,u)\right)$ for fixed $N$ and $L$. However, the GE does not since the sum of the second and third terms on the right-hand side of \eqref{eq:tgebound} decays as $\mathcal{O}(1/\sqrt{\min{(m,u)}}\,)$. Ideally, the sizes of the training and testing datasets should be comparable for the TRC to scale well with $n$. Concerning the matrix size, the bound shows that increasing $N$ or $L$ results in a larger TRC bound regardless of the number of data points $n$. Moreover, the nuclear norm of a matrix is $\mathcal{O}(\sqrt{NL})$ since $\frobs{\F}\leq\nuclear{\F}\leq \sqrt{r}\frobs{\F}$. Therefore, $t$ should also scale with $N$ and $L$ in order to match the hypothesis class, and obtain a good estimate of $\F$. 
\vspace{-0.2cm}
\subsection{Rademacher complexity for KMC}
Unlike base MC that maximizes the nuclear norm of the data matrix, KMC does not directly employ the rank in its objective function. Instead, it imposes constraints on the maximum norm of the factor matrices in their respective RKHSs. Similar to~\cite{shamir2014}, the TRC for KMC is bounded as follows. 
\begin{theorem}\label{th:kmc} 
	If the KMC hypothesis class is $\mF_K:=\left\{\F\right.:$ $\left. \F=\K_w\B\C^T\K_h, \trace{\B^T\K_w\B}  \!+\! \trace{\C^T\K_h\C} \!<\! t_B\right\}$, then 
	\vspace{-0.0cm}
    %------------------------------------------------------------------------------
    \begin{equation}\\[-0.0cm]
      R_n(\mF_K)\leq \lambda_{\max} Gqt_B(\sqrt{N}+\sqrt{L})
    \end{equation}
    %------------------------------------------------------------------------------
    where $\lambda_{\max}$ is the largest eigenvalue of $\K_w$ and $\K_h$.
\end{theorem}
\begin{proof} Rewrite the nuclear norm in \eqref{eq:trcmcbound} in terms of the KMC constraint as
%------------------------------------------------------------------------------
\begin{align}\label{eq:mcvskmc}
&\hspace*{-0.4cm}\nuclear{\F} \!=\! {1\over 2}(\frob{\W}\!+\!\frob{\H}) \!=\! {1\over 2}(\trace{\B^T\!\K_w^2\B} \!+\! \trace{\C^T\!\K_h^2\C}) \nonumber\\&\leq {\lambda_{\max}\over 2}[\trace{\B^T\K_w\B} +\trace{\C^T\K_h\C}] \leq {\lambda_{\max}t_B\over 2} 
\end{align}
%------------------------------------------------------------------------------
where we used that $\trace{\B^T\K_w^2\B} = \sum_{i=1}^N \b_i^T\K_w^2\b_i$ with $\b_i$ denoting the $i^{th}$ column of $\B$, and $\b_i^T\K_w^{1\over 2}\K_w\K_w^{1\over 2}\b_i \leq \lambda_{\max} \b_i^T\K_w\b_i$. \end{proof}
Theorem \ref{th:kmc} establishes that the TRC bound expressions of KMC and MC are identical within a scale. With $t_B=t$, $\lambda_{\max}$ controls whether KMC has a larger or smaller TRC bound than MC. Thus, according to Theorem~\ref{th:kmc}, the GE bound for KMC shrinks with $n$ and grows with $N,~L$ and $\lambda_{\max}$. 

Interestingly, we will show next that it is possible to have a TRC bound that does not depend on the matrix size. 

Consider the factorizations $\K_w=\bphi_w\bphi_w^T$ and $\K_h=\bphi_h\bphi_h^T$, where $\bphi_w\in\bbR^{N\times d_w}$ and $\bphi_h\in\bbR^{N\times d_h}$. Plugging the latter into the objective of \eqref{eq:kmc} and substituting $\W=\K_w\B$ and $\H=\K_h\C$, yields
\begin{align}
    &\frob{P_{\mS_m}(\M-\bphi_w\bphi_w^T\B\C^T\bphi_h\bphi_h^T)} \!+ \mu\!\left(\trace{\B^T\bphi_w\bphi_w^T\B} \right.\nonumber\\&\left. + \trace{\C^T\bphi_h\bphi_h^T\C}\right) \label{eq:kmcphi} \\ 
    &\!=\!\frob{P_{\mS_m}\!(\M\!-\!\bphi_w\A_w\A_h^T\bphi_h^T)\!} \!+\! \mu\!\left(\!\frob{\A_w}\!+\!\frob{\A_h}\!\right)\label{eq:kmcphi2}
\end{align}
where $\A_w = \bphi_w^T\B$ and $\A_h = \bphi_h^T\C$ are coefficient matrices of size $d_w\times p$ and $d_h \times p$, respectively. Optimizing for $\{\B,\C\}$ in \eqref{eq:kmcphi} or for $\{\A_w,\A_h\}$ in \eqref{eq:kmcphi2} yields the same $\hat{\F}$ provided that $\{\bphi_w^T,\bphi_h^T\}$ have full column rank. Under this assumption, we consider the hypothesis class $\mF_I:=\left\{\F:
\F=\bphi_w\A_w\A_h^T\bphi_h^T, \frob{\A_w} \leq t_w, \frob{\A_h} < t_h\right\}$, which satisfies $\mF_I=\mF_K$. Clearly,  \eqref{eq:kmcphi2} is the objective used by the inductive MC~\cite{jain2013}; and therefore, we have shown that inductive MC is a special case of KMC. This leads to the following result. 
\begin{theorem}\label{th:kmclin} If $\K=(\bphi_h\otimes\bphi_w)(\bphi_h\otimes\bphi_w)^T$, and $\S_n$ is a binary sampling matrix that selects the entries in $\mS_n$, then
    %------------------------------------------------------------------------------
    \begin{equation}
        R_n(\mF_I)\leq q\sqrt{t_wt_h}\trace{\S_n\K\S_n^T}.
    \end{equation}
    %------------------------------------------------------------------------------
\end{theorem}
\begin{proof} With $\bm \sigma := \text{vec}(\Sig)$, $b_w :=\frob{\A_w}$, and $b_h:=\frob{\A_h}$, we have that 
	\vspace{-0.1cm} 
    %------------------------------------------------------------------------------
    \begin{align}
        &R_n(\mF_I) =q\es{\hspace{-0.5cm}\sup_{\hspace{0.5cm}\substack{b_w\leq t_w,b_h\leq t_h}}\hspace{-0.5cm}\bm \sigma^T \text{vec}(\bphi_w\A_w\A_h^T\bphi_h^T)} \nonumber \\ \nonumber\\[-0.7cm]&
        = q\es{\hspace{-0.5cm}\sup_{\hspace{0.5cm}\substack{b_w\leq t_w,b_h\leq t_h}}\hspace{-0.5cm}\bm \sigma^T (\bphi_h\otimes\bphi_w)\text{vec}(\A_w\A_h^T)} \nonumber \\
        &\leq q\es{\hspace{-0.5cm}\sup_{\hspace{0.5cm}\substack{b_w\leq t_w,b_h\leq t_h}}\hspace{-0.5cm}\norms{\bm\sigma^T (\bphi_h\otimes\bphi_w)} \norms{\text{vec}(\A_w\A_h^T)}}  \nonumber\\
        &= q\es{\hspace{-0.5cm}\sup_{\hspace{0.5cm}\substack{b_w\leq t_w,b_h\leq t_h}}\hspace{-0.5cm}\sqrt{\bm\sigma^T \K\bm\sigma} \frobs{\A_w\A_h^T}} \nonumber \\
    &\leq q\es{\hspace{-0.5cm}\sup_{\hspace{0.5cm}\substack{b_w\leq t_w,b_h\leq t_h}}\hspace{-0.5cm}\sqrt{\bm\sigma^T \K\bm\sigma} \frobs{\A_w}\frobs{\A_h^T}} \nonumber\\& 
    \leq q\sqrt{t_wt_h}\sqrt{\es{\bm\sigma^T \K\bm\sigma}} 
= q\sqrt{t_wt_h}\sqrt{\trace{\S_n\K\S_n^T}}\nonumber\\[-0.7cm] \nonumber
    \end{align}
    %------------------------------------------------------------------------------
where we have successively used the Cauchy-Schwarz inequality, the sub-multiplicative property of the Frobenius norm, and Jensen's inequality in the first, second and third inequalities, respectively.
    \end{proof}\vspace{-0.1cm}
If entries in the diagonal of $\K$ are bounded by a constant, and $m=u$, Theorem~\ref{th:kmclin} provides a bound that decays as $\mathcal{O}({\sqrt{t_wt_h\over m}})$. Thus, if $t_w$ and $t_h$ are constant, the bound does not grow with $N$ or $L$. These values can reasonably be kept constant when the coefficients in $\{\A_w,\A_h\}$ are not expected to change much as new rows or columns are added to $\F$, e.g., when the existing entries in the kernel matrices are largely unchanged as the matrices grow. For instance, let us rewrite the loss in~\eqref{eq:kmcphi2} as $\norm{\overline{\m}-\S(\bphi_h\otimes\bphi_w)\text{vec}(\A_w\A_h)}$. If when increasing $N$ or $L$ we add a few rows to $\bphi_w$ or $\bphi_h$, as it would have happened with a linear kernel, optimizing for $\{\A_w,\A_h\}$ in~\eqref{eq:kmcphi2} should yield similar results as with smaller $N$ and $L$, as long as the space spanned by $\S(\bphi_h\otimes\bphi_w)$ is not significantly altered.

\subsection{Rademacher complexity for KKMCEX}
In KKMCEX, the restriction is set on the magnitude of $\bgam^T\bK_f\bgam$, which depends on $\S$. Therefore, the hypothesis class for \eqref{eq:kkmcexrep} is not altered by changes in the matrix size. The TRC bound is then given by the next theorem. 
\begin{theorem}\label{th:trckkmcex} If $\mF_R:=\{\F\!:\!\F=\text{\normalfont unvec}(\K_f\S^T\bgam), \bgam^T\bK_f\bgam \leq b^2,\: b\in\bbR\}$ is the hypothesis class for KKMCEX, it holds that
	\vspace{-0.0cm}
    %------------------------------------------------------------------------------
    \begin{equation}\label{eq:trckkmcex}
    R_n(\mF_R) \leq qb\sqrt{\trace{\S_n\K_f\S^T\bK_f^{-1}\S\K_f\S_n^T}}.
    \end{equation}
    %------------------------------------------------------------------------------
\end{theorem}\vspace{-0.0cm}
\begin{proof}
%------------------------------------------------------------------------------
    \begin{align}\nonumber\\[-1.0cm]
    &R_n(\mF_R) =q\es{\sup_{\bgam^T\K_f\bgam\leq b}\bm\sigma^T\K_f\S^T\bgam} \nonumber\\&
    = q\es{\sup_{\bgam^T\bK_f\bgam\leq b}\bm\sigma^T\K_f\S^T\bK_f^{-{1\over2}}\bK_f^{1\over2}\bgam} \nonumber\\
    &\leq q\es{\sup_{\bgam^T\bK_f\bgam\leq b}\norms{\bm\sigma^T\K_f\S^T\bK_f^{-1\over2}}\norms{\bK_f^{1\over2}\bgam}} \nonumber\\&
    \leq qb\es{\norms{\bm\sigma^T\K_f\S^T\bK_f^{- {1\over2}}}} \nonumber\\
&= qb\sqrt{\trace{\S_n\K_f\S\bK_f^{-1}\S^T\K_f\S_n^T}}.\\[-0.95cm]\nonumber
\end{align}
%------------------------------------------------------------------------------    
    \end{proof}\vspace{-0.25cm}
    Supposing that the diagonal entries of $\K_f$ are bounded by a constant, the bound in \eqref{eq:trckkmcex} decays as $\mathcal{O}(\sqrt{n}/\min(m,u))$. For $m=u$, this yields a rate $\mathcal{O}({1\over\sqrt{m}})$. Thus, the GE bound induced by \eqref{eq:trckkmcex} only scales with the number of samples. As a result, we can expect the same performance on the testing dataset regardless of the data matrix size. Moreover, thanks to its simplicity and speed~\cite{gimenez2018}, KKMCEX can be used to confidently initialize other algorithms when needed, e.g., when the prior information is not accurate enough to provide a reliable hypothesis space.

\section{Numerical tests}

%%%%%%%%%%%%%%%%%%%%%%%%%%%%%%%%%%%%%%%%%%%%%%%%%%%%%
\begin{figure*}[!ht]
\vspace{-0.5cm}
\centering
    \begin{minipage}{.58\textwidth}
% This file was created by matplotlib2tikz v0.7.2.
\begin{tikzpicture}
    \vspace{-0.5cm}

\definecolor{color0}{rgb}{0.12156862745098,0.466666666666667,0.705882352941177}
\definecolor{color1}{rgb}{1,0.498039215686275,0.0549019607843137}
\definecolor{color2}{rgb}{0.172549019607843,0.627450980392157,0.172549019607843}
\definecolor{color3}{rgb}{0.83921568627451,0.152941176470588,0.156862745098039}
\definecolor{color4}{rgb}{0.580392156862745,0.403921568627451,0.741176470588235}
\definecolor{color5}{rgb}{0.549019607843137,0.337254901960784,0.294117647058824}
\definecolor{color6}{rgb}{0.890196078431372,0.466666666666667,0.76078431372549}
\definecolor{color7}{rgb}{0.737254901960784,0.741176470588235,0.133333333333333}
\pgfplotsset{scaled y ticks=false}

\begin{groupplot}[
group style={
        group name=my fancy plots,
        group size=1 by 2,
        xticklabels at=edge bottom,
        vertical sep=0pt
},
legend cell align={left},
legend style={at={(1.48,1)}, anchor=east, draw=white!80.0!black},
tick align=outside,
tick pos=left,
width=\figW,
x grid style={white!69.01960784313725!black},
xmajorgrids,
xmin=100, xmax=3100,
y grid style={white!69.01960784313725!black},
ymajorgrids,
ymin=-0, ymax=0.582079421628479,
axis line style={-},
xtick={100,  700,  1300, 1900, 2500, 3100},
]
\nextgroupplot[ymin=0.1,ymax=0.582079421628479,height=3.5cm,axis x line=bottom,axis y discontinuity=parallel,
ytick={-0.1,0,0.2,0.3,0.4,0.5},
yticklabels={−0.1,0.0,0.2,0.3,0.4,0.5},
ylabel={Error},
ylabel style={at={(ticklabel* cs:0.0)}, yshift=0.1cm},
x axis line style={draw=none},
axis line style={-},
xtick={100,  700,  1300, 1900, 2500, 3100},
]
\addplot [thick, color0, mark=x, mark size=3, mark options={solid,opacity=1}]
table {%
%100 0.0243523679158521
400 0.216592777756433
700 0.310872016319012
1000 0.356618487758565
1300 0.380469679159684
1600 0.39495332044335
1900 0.410331383059142
2200 0.411498522401412
2500 0.419835471541626
2800 0.430092379071887
3100 0.438105389519724
};
%\addlegendentry{MC test}
\addplot [thick, color1, mark=diamond*, mark size=2, mark options={solid,opacity=0.5}]
table {%
%100 0.000456957072233675
400 0.00224407246303645
700 0.00367243960472542
1000 0.00470199925231404
1300 0.00534247800436177
1600 0.00594615460567443
1900 0.00630332658914948
2200 0.00665208666976128
2500 0.00689931811979979
2800 0.00708265830859615
3100 0.00725908751571346
};
%\addlegendentry{MC train}
\addplot [thick, color2, mark=square*, mark size=2, mark options={solid,opacity=0.5}]
table {%
%100 0.0238954108436184
400 0.214348705293397
700 0.307199576714286
1000 0.351916488506251
1300 0.375127201155322
1600 0.389007165837675
1900 0.404028056469993
2200 0.404846435731651
2500 0.412936153421827
2800 0.423009720763291
3100 0.43084630200401
};
%\addlegendentry{MC gen}
\addplot [thick, color3, dash pattern=on 1pt off 3pt on 3pt off 3pt, mark=x, mark size=3, mark options={solid,opacity=1}]
table {%
100 0.00129917675926872
400 0.00542177774410175
700 0.00626109438833945
1000 0.00894711970310475
1300 0.012262840082126
1600 0.0144995044205012
1900 0.0212640554348226
2200 0.0267341943758068
2500 0.029514054244211
2800 0.030019226074137
3100 0.0346975638900754
};
\nextgroupplot[ymin=-0.003, axis x line=bottom, ymax=0.04,height=3.5cm,
ytick={0,0.01,0.02,0.02,0.03,0.04,0.05},
yticklabels={0.00,0.01,~,0.02,0.03,0.04,0.05},
xlabel={$N$~~~(a)},
xlabel style={yshift=0.09cm},
axis line style={-},
xtick={100,  700,  1300, 1900, 2500, 3100},
]
\addplot [thick, color2, mark=square*, mark size=2, mark options={solid,opacity=0.5}]
table {%
100 0.0238954108436184
%400 0.214348705293397
%700 0.307199576714286
%1000 0.351916488506251
%1300 0.375127201155322
%1600 0.389007165837675
%1900 0.404028056469993
%2200 0.404846435731651
%2500 0.412936153421827
%2800 0.423009720763291
%3100 0.43084630200401
};
\addlegendentry{MC gen.}
\addplot [thick, color0, mark=x, mark size=3, mark options={solid,opacity=1}]
table {%
100 0.0243523679158521
%400 0.216592777756433
%700 0.310872016319012
%1000 0.356618487758565
%1300 0.380469679159684
%1600 0.39495332044335
%1900 0.410331383059142
%2200 0.411498522401412
%2500 0.419835471541626
%2800 0.430092379071887
%3100 0.438105389519724
};
\addlegendentry{MC test}
\addplot [thick, color1, mark=diamond*, mark size=3, mark options={solid,opacity=0.5}]
table {%
100 0.000456957072233675
400 0.00224407246303645
700 0.00367243960472542
1000 0.00470199925231404
1300 0.00534247800436177
1600 0.00594615460567443
1900 0.00630332658914948
2200 0.00665208666976128
2500 0.00689931811979979
2800 0.00708265830859615
3100 0.00725908751571346
};
\addlegendentry{MC train}

\addplot [very thick, color5, dotted, mark=square*, mark size=2, mark options={solid,opacity=0.5}]
table {%
100 0.00124457435591168
400 0.00542107769376893
700 0.00623879631750175
1000 0.00892336314175231
1300 0.0122372468721463
1600 0.0144735200012541
1900 0.0212634700106334
2200 0.026733503484668
2500 0.0295134242302942
2800 0.0299817327161369
3100 0.0346968266582423
};
\addlegendentry{KMC gen.}
\addplot [very thick, color3, dotted, mark=x, mark size=3, mark options={solid,opacity=1,thick}]
table {%
100 0.00129917675926872
400 0.00542177774410175
700 0.00626109438833945
1000 0.00894711970310475
1300 0.012262840082126
1600 0.0144995044205012
1900 0.0212640554348226
2200 0.0267341943758068
2500 0.029514054244211
2800 0.030019226074137
3100 0.0346975638900754
};
\addlegendentry{KMC test}
\addplot [very thick, color4, dotted, mark=diamond*, mark size=2, mark options={solid,opacity=0.5}]
table {%
100 5.4602403357033e-05
400 7.00050332815843e-07
700 2.22980708377081e-05
1000 2.37565613524389e-05
1300 2.55932099797434e-05
1600 2.59844192471436e-05
1900 5.85424189225362e-07
2200 6.90891138760487e-07
2500 6.30013916788019e-07
2800 3.74933580001167e-05
3100 7.37231833152606e-07
};
\addlegendentry{KMC train}

\addplot [thick, color7, dashed, mark=square*, mark size=2, mark options={solid,opacity=0.5}]
table {%
100 0.00207860701877824
400 0.00357613147151784
700 0.00394133002403142
1000 0.00404426871868387
1300 0.00417910414956755
1600 0.00426354254302273
1900 0.00427194227519393
2200 0.00443938119670363
2500 0.00446285956382869
2800 0.0043729799107388
3100 0.00437610763034754
};
\addlegendentry{KKMCEX gen.}
\addplot [thick, color6, dashed, mark=x, mark size=3, mark options={solid,opacity=1}]
table {%
100 0.00207860701877824
400 0.00357613147151784
700 0.00394133002403142
1000 0.00404426871868387
1300 0.00417910414956755
1600 0.00426354254302273
1900 0.00427194227519393
2200 0.00443938119670363
2500 0.00446285956382869
2800 0.0043729799107388
3100 0.00437610763034754
};
\addlegendentry{KKMCEX test}
\addplot [thick, white!49.80392156862745!black, dashed, mark=diamond*, mark size=2, mark options={solid,opacity=0.5}]
table {%
100 8.39884225123272e-24
400 4.6844619998516e-24
700 3.99351146671556e-24
1000 3.59125643590348e-24
1300 3.42531276495896e-24
1600 3.29755205789112e-24
1900 3.1632712201253e-24
2200 3.18424217079255e-24
2500 3.0883076519882e-24
2800 2.96013846690785e-24
3100 2.92250676940605e-24
};
\addlegendentry{KKMCEX train}%
\end{groupplot}%

\end{tikzpicture}
    \end{minipage}%
    \begin{minipage}{.43\textwidth}
% This file was created by matplotlib2tikz v0.7.2.
\begin{tikzpicture}

\definecolor{color0}{rgb}{0.12156862745098,0.466666666666667,0.705882352941177}
\definecolor{color1}{rgb}{1,0.498039215686275,0.0549019607843137}
\definecolor{color2}{rgb}{0.172549019607843,0.627450980392157,0.172549019607843}
\definecolor{color3}{rgb}{0.83921568627451,0.152941176470588,0.156862745098039}
\definecolor{color4}{rgb}{0.580392156862745,0.403921568627451,0.741176470588235}
\definecolor{color5}{rgb}{0.549019607843137,0.337254901960784,0.294117647058824}
\definecolor{color6}{rgb}{0.890196078431372,0.466666666666667,0.76078431372549}
\definecolor{color7}{rgb}{0.737254901960784,0.741176470588235,0.133333333333333}

\begin{axis}[
height=5.4cm,
tick align=outside,
tick pos=left,
width=\figW,
x grid style={white!69.01960784313725!black},
xlabel={$N$~~~(b)},
xmajorgrids,
xmin=-50, xmax=3250,
y grid style={white!69.01960784313725!black},
ylabel={Error},
ymajorgrids,
ymin=-0.00, ymax=0.582079421628479,
xmin=100, xmax=3100,
ytick={-0.1,0,0.1,0.2,0.3,0.4,0.5,0.6},
yticklabels={−0.1,0.0,0.1,0.2,0.3,0.4,0.5,0.6},
xtick={100,  700,  1300, 1900, 2500, 3100},
axis x line = bottom,
axis line style={-},
]

\addplot [thick, color0, mark=x, mark size=3, mark options={solid,opacity=1}]
table {%
100 0.228334842727579
400 0.435250454082663
700 0.460368788972204
1000 0.498035574571511
1300 0.513913025490087
1600 0.529404893854791
1900 0.537367936508457
2200 0.546416773665449
2500 0.538881201643675
2800 0.533685022669763
3100 0.554385422616032
};
\addplot [thick, color1, mark=diamond*, mark size=2, mark options={solid,opacity=0.5}]
table {%
100 0.000505442367094315
400 0.00233190845085261
700 0.00353935521191725
1000 0.00453632473239951
1300 0.00505362332215025
1600 0.00559597815068394
1900 0.00611501774602094
2200 0.00636901984579126
2500 0.00660755969415787
2800 0.00675902127013661
3100 0.0068593727208381
};
\addplot [thick, color2, mark=square*, mark size=2, mark options={solid,opacity=0.5}]
table {%
100 0.227829400360485
400 0.432918545631811
700 0.456829433760286
1000 0.493499249839111
1300 0.508859402167937
1600 0.523808915704107
1900 0.531252918762437
2200 0.540047753819658
2500 0.532273641949517
2800 0.526926001399627
3100 0.547526049895194
};
\addplot [very thick, color3, dotted, mark=x, mark size=3, mark options={solid,opacity=1,thick}]
table {%
100 0.129785065488631
400 0.141193599352017
700 0.14892716663038
1000 0.15914103365446
1300 0.162886842441053
1600 0.168740410169324
1900 0.164657857568993
2200 0.165760050056446
2500 0.167880009634629
2800 0.16960650346207
3100 0.174721246397494
};
\addplot [very thick, color4, dotted, mark=diamond*, mark size=2, mark options={solid,opacity=0.5}]
table {%
100 0.0892240179289425
400 0.0718918874575301
700 0.0695607563057266
1000 0.0732806967278745
1300 0.073246853879856
1600 0.0761755724928124
1900 0.0734581576633236
2200 0.0738458010861975
2500 0.0758234973725545
2800 0.0772508789468431
3100 0.0780244497847436
};
\addplot [very thick, color5, dotted, mark=square*, mark size=2, mark options={solid,opacity=0.5}]
table {%
100 0.0405610475596888
400 0.0693017118944872
700 0.0793664103246538
1000 0.0858603369265852
1300 0.0896399885611973
1600 0.0925648376765113
1900 0.0911996999056694
2200 0.0919142489702486
2500 0.0920565122620748
2800 0.0923556245152265
3100 0.0966967966127506
};
\addplot [thick, color6, dashed, mark=x, mark size=3, mark options={solid,opacity=1}]
table {%
100 0.128392017425107
400 0.130857831677917
700 0.129719180483348
1000 0.134012388125453
1300 0.133540297159052
1600 0.13428009125326
1900 0.134337521649453
2200 0.135175040933584
2500 0.132905436882377
2800 0.131919245304758
3100 0.134649038129618
};
\addplot [thick, white!49.80392156862745!black, dashed, mark=diamond*, mark size=2, mark options={solid,opacity=0.5}]
table {%
100 0.080479329186886
400 0.0753146622904499
700 0.073753563443254
1000 0.0757182542244422
1300 0.0756740237058565
1600 0.0755832509586624
1900 0.0758303001417337
2200 0.0761839636654129
2500 0.0751179924480385
2800 0.0744889309409899
3100 0.0757991971936638
};
\addplot [thick, color7, dashed, mark=square*, mark size=2, mark options={solid,opacity=0.5}]
table {%
100 0.0479126882382207
400 0.0555431693874666
700 0.0559656170400936
1000 0.0582941339010106
1300 0.0578662734531954
1600 0.0586968402945977
1900 0.0585072215077194
2200 0.0589910772681714
2500 0.057787444434339
2800 0.0574303143637681
3100 0.0588498409359544
};
\end{axis}

\end{tikzpicture}
    \end{minipage}%
\vspace{-0.6cm}
\caption{Training loss, testing loss, and generalization error vs. matrix size for: (a) $snr=\infty$, and (b) $snr=4$.}
\label{fig:noisy}
\end{figure*}
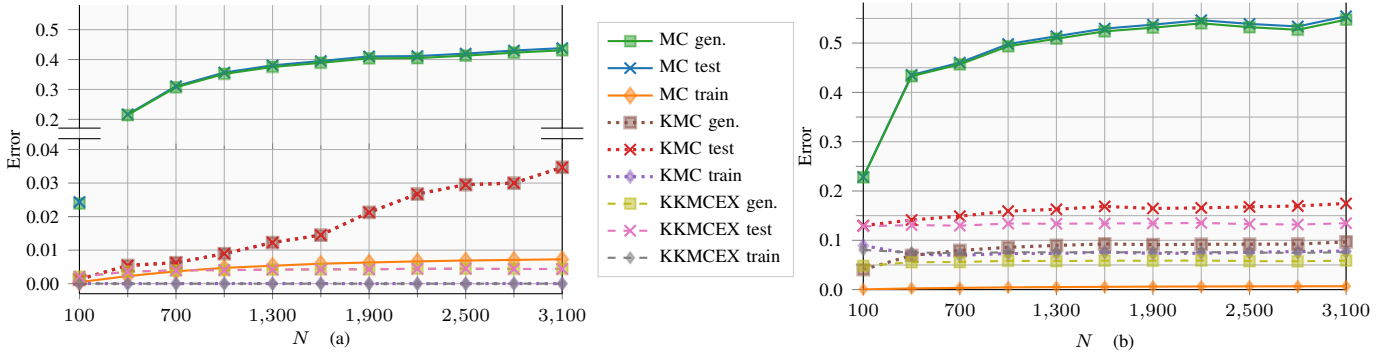
%%%%%%%%%%%%%%%%%%%%%%%%%%%%%%%%%%%%%%%%%%%%%%%%%%%%%
This section compares the GE of MC and KMC, solved via alternating least-squares (ALS)~\cite{jain2013low}, with the KKMCEX solved with \eqref{eq:kkmcexsol}. Besides comparing the GE of these algorithms, we also assess how the matrix size impacts the GE. To this end, we use a fixed-rank synthetic data matrix with $N=L$ generated as $\F = \K_w\B\C^T\K_h + \E$. The kernel matrices are $\K_w=\K_h=\text{abs}(\R\D\R^T)$, where $\R \in\mathbb{C}^{N\times N}$ is the DFT basis and $\D\in\mathbb{R}^{N\times N}$ is a diagonal matrix with decreasing weights on its diagonal. The coefficient matrices $\{\B,\C\}$ have $p=30$ columns, with  entries drawn from a zero-mean Gaussian distribution with variance 1. The entries of $\E\in\mathbb{R}^{N\times N}$ are drawn from a zero-mean Gaussian distribution with variance set according to the signal-to-noise ratio $snr=\frob{\F}/\frob{\E}$.

The tests are run over 1,000 realizations. A new matrix $\F$ is generated per realization with $m\!=\!1,000$ entries drawn uniformly at random, followed by a run of each algorithm. Then, the loss on the testing set, which consists of the remaining $u=N^2-m$ entries, is measured. A single value of $\mu$ chosen by cross-validation is used for all realizations. For KMC and KKMCEX, $\mu$ is scaled with the matrix size to compensate for the trace growth of the kernel matrices, and thus keep the loss and regularization terms balanced.

Fig.~\ref{fig:noisy}a shows the training, testing, and GEs for square matrices with size ranging from $N=100$ to $N=3,200$, and $snr = \infty$. We observe for base MC that the training loss is small, whereas it is much larger on the testing dataset, and also it grows with $N$. Moreover, since the training loss is minimal, the GE coincides with the testing loss. Clearly, the MC solution \eqref{eq:mclag} is not able to predict the unobserved entries due to the lack of prior information that would allow for extrapolation. In addition, the GE approaches saturation for large matrix sizes since most entries in the estimated matrix are $0$, and the testing loss tends to the average ${1\over u}\sum_{(i,j)\in\mS_u}\M_{i,j}$. Regarding the performance of KMC and KKMCEX, we observe that both algorithms achieve a constant training loss. Although not visible on the plot, the training loss of KKMCEX is one order of magnitude smaller than that of KMC. On the other hand, the testing and GE of KKMCEX are constant unlike in KMC for which both are higher and grow with $N$. These results confirm what was asserted by the TRC bounds in Section~\ref{sec:radbounds}.

Fig.~\ref{fig:noisy}b shows the same simulation results as Fig.~\ref{fig:noisy}a, but with noisy data at $snr=4$. We observe that MC overfits the noisy observations since the training loss is, again, very small, while the testing loss is much larger. For KMC and KKMCEX, the presence of noise increases the training and testing losses. Due to the noise, a larger $\mu$ is selected to prevent overfitting at the cost of a higher training loss.  Nevertheless, the testing loss of KMC slightly grows with $N$. In terms of GE, KKMCEX outperforms KMC with a lower value that tends to a constant. 

{\bibliographystyle{IEEEtran}

}
 \end{document}